\documentclass{article}

\PassOptionsToPackage{numbers, compress, sort}{natbib}

\usepackage[preprint]{neurips_2024}

\usepackage[utf8]{inputenc} 
\usepackage[T1]{fontenc}    
\usepackage{hyperref}       
\usepackage{url}            
\usepackage{booktabs}       
\usepackage{amsfonts}       
\usepackage{nicefrac}       
\usepackage{microtype}      
\usepackage{xcolor}         

\author{%
  Evelyn Ma
    \\
  University of Illinois at Urbana-Champaign\\
  \texttt{pingm@illinois.edu} \\
  \And
  Chao Pan \\
  University of Illinois at Urbana-Champaign \\
  \texttt{chaopan2@illinois.edu} \\
  \AND
  Rasoul Etesami \\
  University of Illinois at Urbana-Champaign \\
  \texttt{etesami1@illinois.edu} \\
  \And
  Han Zhao \\
  University of Illinois at Urbana-Champaign \\
  \texttt{hanzhao@illinois.edu} \\
  \And
  Olgica Milenkovic \\
  University of Illinois at Urbana-Champaign \\
  \texttt{milenkov@illinois.edu} \\
}

\usepackage{enumitem}

\usepackage{amsmath}
\usepackage{tcolorbox,xcolor}
\usepackage{amssymb}
\usepackage{amsthm}
\usepackage{mathrsfs}
\usepackage{makecell}
\usepackage{hyperref}
\usepackage{aisecure-math}
\usepackage{algpseudocode}
\usepackage{algorithm}
\usepackage{multirow}

\usepackage{pifont}

\newcommand{\blue}[1]{\textcolor{black}{#1}}

\begin{document}

\title{FedGTST: Boosting Global Transferability of Federated Models via Statistics Tuning}

\maketitle


\begin{abstract}
  The performance of Transfer Learning (TL) significantly depends on effective pretraining, which not only requires extensive amounts of data but also substantial computational resources. As a result, in practice, it is challenging to successfully perform TL at the level of individual model developers.
  Federated Learning (FL) addresses these challenges by enabling collaborations among individual clients through an indirect expansion of the available dataset, distribution of the computational burden across different entities, and privacy-preserving communication mechanisms.
  Despite several attempts to devise effective transferable FL approaches, several important issues remain unsolved.
  First, existing methods primarily focus on optimizing transferability within local client domains, thereby ignoring transferability over the global learning domain. Second, most approaches focus on analyzing indirect transferability metrics, which does not allow for accurate assessment of the final target loss and the degree of transferability.
  To address these issues, we introduce two important FL features into the model. The first boosts transferability via an exchange protocol between the clients and the server that includes information about cross-client Jacobian (gradient) norms. The second feature promotes an increase of the average of the Jacobians of the clients at the server side, which is subsequently used as a local regularizer that reduces the cross-client Jacobian variance.
  A rigorous analysis of our transferable federated algorithm, termed {FedGTST} (Federated Global Transferability via Statistics Tuning), reveals that increasing the averaged Jacobian norm across clients and reducing the Jacobian variance ensures tight control of the target loss. This insight leads to an upper bound on the target loss of transferable FL in terms of the source loss and source-target domain discrepancy.
  Extensive experimental results on datasets including MNIST $\to$ MNIST-M and CIFAR10 $\to$ SVHN suggest that FedGTST significantly outperforms other relevant baselines, such as FedSR.
  For example, on the second source-target dataset pair, we improve the accuracy of FedSR by $9.8\%$ and that of FedIIR by $7.6\%$ when the backbone used is LeNet.
\end{abstract}

\section{Introduction}
\label{sec:intro}
\emph{Transfer Learning} (TL) has received significant interest in the machine learning community due to its ability to extract representative features from source tasks and use them to improve the generalization capability on related target domain problems~\cite{yosinski2014transferable, weiss2016survey}. In addition to boosting the performance of a target domain model, TL also reduces the computational cost of fine-tuning the target domain model. Nevertheless, effective source pretraining in TL is practically challenging for individual model developers because it requires access to large datasets as well as significant computational resources~\cite{pan2010survey}.
To resolve this problem, one can leverage \emph{Federated Learning} (FL), which refers to decentralized learning protocols used in mobile and IoT devices~\cite{konevcny2016federated}. FL not only increases access to multiple datasets in a decentralized manner and alleviates the computational burden of individual clients, but it also protects the privacy of local data ~\cite{mcmahan2017communication}. As a result, a number of recent works have outlined methods for transferable FL, including FedADG (Federated Adversarial Domain Generalization)~\cite{zhang2021federated}, FedCDG (Federated Contrastive Domain Generalization)~\cite{yu2023contrastive}, FedSR (Federated Simple Representations)~\cite{nguyen2022fedsr}, FedIIR (Federated Implicit Invariant Relationships)~\cite{guo2023out}, FedCCST (Federated Cross-Client Style Transfer)~\cite{chen2023federated} and StableFDG (Stable Federated Domain Generalization)~\cite{park2024stablefdg}.
Despite the promising initial findings provided by the aforementioned techniques, several combinations of important issues remain unsolved across the spectrum of methods.

With respect to \textbf{privacy leakage}, the drawabacks of a collection of the methods are as follows:
\textbf{1.} FedADG forces each client source domain to align its representation distribution with that of the target domain, and therefore violates data privacy because source domains are given access to the target domain in order to perform the alignment; \textbf{2.} FedCCST boosts global transferability by increasing local diversity to avoid local overfitting. It therefore requires clients to share their local representations with each other and this information is subsequently used for local data augmentation. This is a direct violation of FL  privacy constraints; \textbf{3.}
StableFDG expands local data diversity by sharing style statistics (i.e., representations, including means and variances). This clearly leads to leakage of local privacy-sensitive information.

With respect to \textbf{local overfitting}, the shortcomings of a group of the aforementioned methods are as follows: \textbf{1.} FedSR learns a simple representation through regularization during  local training, by exploiting the similarity between the representation and the data, given the labels. However, since the regularized local training relies completely on local structures (i.e., local models, representations, labels, data), it leads to overfitting of local distributions, and thus has limited capability to learn cross-client invariant features, which is key for global transferability; \textbf{2.}  FedCDG uses a contrastive local regularizer on representations generated by various samples within the same class. This leads to overfitting in the local domain since no cross-client information is exploited.

With respect to \textbf{communication complexity}, we observe that: \textbf{1.} FedIIR is  suboptimal. Although it  mitigates the problem of privacy violation and avoids local overfitting by adding a local regularizer capturing the distance between the local gradient and the global gradient, it requires communicating gradients between the clients and the server and therefore doubles the communication cost compared to baselines (additionally, FedIIR performs well for a large number of clients, but offers average performance when this number is small); \textbf{2.} Similar communication complexity problems are faced by FedCCST and StableFDG, which rely on communicating styles (i.e., representations).

Finally, prior works mostly \textbf{lack explicit theoretical analyses of global transferability:} they do not tend to quantify the performance/loss  of the pretrained model fine-tuned on the target domain.

In summary, the most important unresolved problem with known transferable FL models (with the exception of FedIIR) is that they use centralized TL approaches during local training, and do not fully exploit features specific to federated learning (for details, see also the discussion in Section~\ref{sec:relatedwork}).

\textbf{Our contributions.} We describe what is, to the best of our knowledge, the first approach to federated transfer learning termed \emph{Federated Global Transferability via Statistics Tuning (FedGTST)} that simultaneously alleviates the above issues faced by existing methods. Our main contributions can be summarized as follows. 

\begin{itemize}
    \item[\textbf{1.}] We use a new regularizer that encodes cross-client statistics and forces the local training process to tune the global statistics in a ``direction'' that improves global transferability rather than just local transferability. This is achieved through subtractions of global norms of Jacobians (gradients) communicated by the server. 
    \item[\textbf{2.}] We suggest to only communicate scalars, more precisely, \emph{Jacobian norms}, which introduces a negligible communication overhead in the overall model exchange protocol.
    \item[\textbf{3.}] We ensure that our communication schemes do not allow uncontrolled access to data and thereby ensure data privacy.
    \item[\textbf{4.}] We rigorously prove that even though only small discrepancies among local gradients may exist upon regularization, transferability can be low as regularization can impede the growth of the gradient norm. To boost the Jacobian norm, we implement specialized protocols at both the client and server levels. Finally, we establish relevant bounds on the transferability loss for this setting. 
\end{itemize}


The main technical insights provided by our analysis are as follows. Two FL-specific factors, a small \textit{cross-client Jacobian variance} and larger \textit{cross-client Jacobian norm} are indicative of good transferability. These factors are \emph{direct} performance indicators, unlike \emph{indirect factors} (e.g., feature invariance) which only suggest improved transferability. Our findings are based on the first \textit{direct measure} of transferability, which equals the loss on the target domain incurred by the pretrained federated model. The FL-specific factors govern the bounds on the loss and therefore allow one to control them for better transferability. We validate these findings through extensive experiments which show that FedGTST outperforms methods such as FedSR and FedIIR by as much as $10\%$.

\section{Related Work}
\label{sec:relatedwork}
\vspace{-0.1in}
\textbf{FL} is a machine learning paradigm in which multiple entities collaborate to train a global model without sharing their local data (see ~\cite{kairouz2019advances} for a comprehensive review). FL has gained significant attention due to its potential to address privacy concerns while enabling large-scale collaborative learning ~\cite{smith2017federated}. 
Relevant to this work,~\cite{mcmahan2017communication} proposed the Federated Averaging (FedAvg) algorithm, which aggregates model updates from multiple client devices to train a global model. Another relevant line of work~\cite{yang2019federated} introduced FedProx, a federated optimization algorithm that incorporates proximal terms to handle non-iid data distributions. FL methods nevertheless still face several challenges. One challenge is dealing with highly heterogeneous local datasets ~\cite{li2020federated}
\blue{, for which the recent work FedImpro\cite{TangZS0L0024} proposed leveraging  aggregated feature distributions to address client drift.} 
Another challenge is the communication overhead incurred during the aggregation of model updates ~\cite{bonawitz2019towards}. Minimizing communication costs while ensuring convergence and data privacy remain active topics of research in FL. 
Also, many FL solutions primarily emphasize performance in the client domain without considering the performance of the model on unseen domains. 


\textbf{TL} is a powerful machine learning technique that allows models to leverage knowledge gained from one task to improve performance on another related task ~\cite{pan2010survey}. TL has been widely adopted in various domains such as computer vision, natural language processing, and speech recognition, where labeled data may be scarce or expensive to acquire~\cite{torrey2010transfer,weiss2016survey}. A common approach in TL involves fine-tuning a pre-trained model on a target task using a small amount of labeled data, which often leads to improved generalization and faster convergence compared to training from scratch ~\cite{yosinski2014transferable}. Recent works in TL have focused on developing more effective  algorithms, such as domain adaptation methods that address the discrepancy between the source and target domains ~\cite{ganin2016domain}. Additionally, TL techniques have been used  to handle tasks with limited amounts of labeled data through techniques like semi-supervised and self-supervised learning ~\cite{ruder2019survey}. TL still faces challenges such as negative transfer, where information from the source task actually degrades performance on the target task; and, it requires careful selecting of appropriate pretrained models and transfer strategies for specific tasks and domains ~\cite{pan2010survey}. 
Current TL methods often require that one entity possesses knowledge of all data, violating the privacy requirements of FL. 
\blue{Moreover, we comment on \textit{Gradient Matching in TL} in Appendix~\ref{subsec:apdx-lit-grad-match}}

\textbf{Transferable Federated Learning} (TFL) is an emerging research area at the intersection of FL and TL. One of earliest contributions to the field, \textit{FedADG}~\cite{zhang2021federated}, encourages the transferablity of FL through adversarial local training. However, the work does not provide theoretical guarantees, and existing studies~\cite{xu2022adversarially} indicate that adversarial robustness does not necessarily lead to better transferability. 
Other methods, such as \textit{FedSR}~\cite{nguyen2022fedsr} and \textit{FedCDG}~\cite{yu2023contrastive}, enhance transferability by adapting standard representation learning from a single-agent to a federated setting; they do not incorporate FL-specific features (i.e., instructions provided by the server, cross-client model properties etc). 
Note that although \textit{FedSR} has successfully included centralized invariant feature learning into FL, it uses centralized methods locally and then shares information with the global model, and thereby does not fully exploiting FL capabilities. Thus, using FedSR, each client can learn very different representations that are hard to aggregate at the central server. More precisely, FedSR does not communicating information that can help improve the transferability of the global model. Among all the previously discussed methods (FedADG, FedCDG,  FedSR, FedIIR,  FedCCST,  and StableFDG), FedIIR is the closest to our approach and may be seen as a special case of our method which has better performance, smaller communication complexity and comes with provable global transferability guarantees. \blue{Furthermore, we discuss the distinctions and connections between TFL and Generalization of FL, a topic potentially relevant to TFL, in Appendix~\ref{subsec:apdx-lit-GFL}.}

\section{Preliminaries}
\label{sec:prelim}
\blue{\textbf{General Supervised Learning Settings}}. We denote the data space by $\gX$, the feature space by $\gZ$, and the label space by $\gY$.
A model $h:\gX\to\gY$ typically takes the form $h = g\circ f$,
where $f: \gX \to \gZ$ is a feature extractor and $g:\gZ \to \gY$ is a classifier.
Denote the function class for the entire model, the feature extractor and the classifier by $\gH, \gF, \gG$, respectively, so that $h\in \gH, f\in \gF, g\in \gG$.
Denote the weights of model $\psi \in \{f,g,h\}$ as $w_\psi$.
Given a loss function $l:\gY \times \gY \rightarrow \sR$ and a domain distribution $\gD$ over $\gX \times \gY$,
the population loss $L_{\gD}(h)$ is defined as
\begin{align}
    \label{eq:population_loss}
    L_{\gD}(h):= \sE_{(x,y) \sim\gD}\;l\left(h(x),y\right).
\end{align}

\blue{\textbf{General Framework of TFL}. In TL, the two typical learning phases are: a) pretraining on the source domain; and, b) finetuning on the target domain. In the context of TFL, pretraining is conducted via FL over source (local) domains, while the global model is trained and then finetuned on the target domain during the second phase. In both phases, supervised learning is performed with full access to the labels. More details are provided next.} 

\blue{\textbf{Pretraining Phase in TFL: FL on Source (Local) Domains}.}
The source domain is a composition of the agents' local domains, $\{\gD^{(k)}\}$, with $k\in [K]$ denoting the client index and $K$ representing the total number of clients. The source loss is defined as the standard federated loss on the source domains.
\vspace{-0.2cm}
\begin{align}
    \label{obj:src_loss}
    L_{src}(g\circ f) & := \frac{1}{K}\sum_{k=1}^K L_{\gD^{(k)}}(g\circ f). 
\end{align}
Let $h^* = g^* \circ f^*$ be an optimal \blue{global} solution for the objective~(\ref{obj:src_loss}). In FL approaches, the problem solution is the result of the central server's aggregation of local models into a global one. We denote the local solutions involved in creating the optimal global solution $\psi^*$ ($\psi \in \{f,g,h\}$) by $\{\psi^{*(k)}\}$; through averaging aggregation, we obtain the optimal global weights  $w^*_\psi = \frac{1}{K}\sum_k w^{*(k)}_\psi$.

\blue{\textbf{Finetuning Phase in TFL: Supervised Finetuning on the Target Domain.} Upon obtaining the optimal pretrained global solution $h^* = g^* \circ f^*$, the pretrained feature extractor $f^*$ is fixed and applied to the target domain $\gD_T$. 
The target loss is defined as the loss on the target domain $\gD_T$, i.e.,}
\begin{align}
    \label{obj:tgt_loss}
    L_{tgt}(g\circ f^*) & := L_{\gD_T}(g\circ f^*).
\end{align}
\blue{Through finetuning,} a new classifier $g^*_T  :=\arg\inf_{g\in \gG} L_{tgt}(g\circ f^{*})$ is determined by minimizing the target objective~(\ref{obj:tgt_loss}).

\blue{\textbf{Transferability Assessment}}.
With a slight abuse of notation, we define the optimal target loss as 
$$
L_{tgt}^* := L_{tgt}(g^*_T \circ f^*).
$$ 
\blue{We formally define the \textit{measure of transferability of TFL} as the optimal target loss \( L_{tgt}^* \), as it \emph{directly} reflects the performance of a transferable model on the target domain. A smaller \( L_{tgt}^* \), or a tighter bound on it, indicates better transferability.}



\section{Theoretical Bounds on the Target Loss}
\label{sec:theory}
\subsection{A General Bound Based on Discrepancy/Divergence}
We start with Definitions~\ref{def:GF_dvg} and ~\ref{def:cc-div} borrowed from existing TL studies that characterize the domain discrepancy, and then propose a new domain divergence tailored to Transferable FL (TFL), including the cross-client discrepancy in Definition~\ref{def:intra-dist} and the source-target discrepancy in~Definition \ref{def:fed_GF_dvg}.

\begin{definition}[$\gG, \gF$-discrepancy~\cite{xu2022adversarially}]
    \label{def:GF_dvg}
    Given a classifier class $\gG$, a feature extractor class $\gF$, the source domain $\gD_S$ and the target domain $\gD_T$, with a slight abuse of notation, the $(\gG, \gF)$-discrepancy is defined as
    \begin{align}
        \label{eq:gf_divergence}
        d_{\gG, \gF}(\gD_S, \gD_T):=
        \sup_{f\in\gF}
        \bigg|\inf_{g\in\gG} L_{\gD^{(k)}}(g \circ f) -
        \inf_{g\in\gG}L_{\gD_T}(g \circ f)\bigg|.
    \end{align}
\end{definition}
{\remark The $\gG,\gF$-discrepancy 
has also been used in the analysis of domain adaptation~\cite{xu2022adversarially}.}

We next adapt the $\gH$-discrepancy to measure the cross-client domain discrepancy in transferable FL (Definition ~\ref{def:cc-div}), and tailor the $\gG, \gF$-discrepancy (Definition ~\ref{def:GF_dvg}) to measure the source--target discrepancy in TFL (Definition ~\ref{def:fed_GF_dvg}).

\begin{definition}[Cross-Client Divengence for TFL]
    \label{def:cc-div}
    Given a model class $\gH$ and federated local domains $\gD_S^{fed} = \{\gD^{(k)}\}_{k\in[K]}$, with $d_{\gH}(\cdot,\cdot)$ defined as Definition \ref{def:h_divergence}, the intra-client discrepancy is defined as
    \label{def:intra-dist}
    \begin{align}
        \overline{d}_{\gH}(\gD_S^{fed}):= \frac{1}{{K(K-1)}}\sum_{k_1\neq k_2} d_{\gH}\left(\gD_S^{(k_1)}, \gD_S^{(k_2)}\right).
    \end{align}
\end{definition}
Note that $\overline{d}_{\gH}(\gD_S)$ equals the average of $\gH$-discrepancies over all local domain pairs and therefore measures the intra-discrepancy on the non-iid distributed source domains. When source domains are iid across clients, we have $\overline{d}_{\gH}(\gD_S)=0$.

\begin{definition}[Source-Target Discrepancy for TFL]
    \label{def:fed_GF_dvg}
    Given a classifier class $\gG$, a feature extractor class $\gF$, federated local domains $\gD_S^{fed} = \{\gD^{(k)}\}_{k\in[K]}$, and the target domain $\gD_T$, with slight abuse of notation, the federated $(\gG, \gF)$-discrepancy is defined as
    \begin{align}
        \label{eq:tfl_gh_divergence}
        d_{\gG, \gF}(\gD_S^{fed}, \gD_T):=
        \frac{1}{K} \sum_{k\in[K]} d_{\gG,\gF}(\gD^{(k)}, \gD_T).
    \end{align}
\end{definition}
Based on the two TFL-specific domain discrepancy definitions (i.e., Definition ~\ref{def:cc-div} and ~\ref{def:fed_GF_dvg}), we derive a general bound on the TFL loss in Theorem \ref{thrm:gen-ub}.
The TFL-specific source-target discrepancy (Definition ~\ref{def:fed_GF_dvg}) is further used in Theorem~\ref{thrm:ub-lr} which presents a bound on the target loss using cross-client statistics. For our theoretical analyses, we need the following common assumptions.

\begin{assumption}[Convexity and Smoothness]
    \label{asp:conv-smooth}
    We assume that the loss function $l$ satisfies two conditions: (1) $l$ is convex w.r.t. $w_{h}$; (2) $l$ is Lipschitz smooth for $w_{h}$ with a constant $\alpha>0$.
\end{assumption}

{\remark  Assumption ~\ref{asp:conv-smooth} is easy to meet in practice, and it arises  in many linear models (linear regression, SVM etc). }

\begin{theorem}[Bound Based on TFL-specific Domain Discrepancy]
    \label{thrm:gen-ub}
    Under Assumptions~\ref{asp:conv-smooth} (Convexity and Smoothness), the optimal target loss is bounded as
    \begin{align}
        \label{eq:gen-bound}
        L^*_{tgt} \leq \frac{1}{K}\sum_{k=1}^K\left[L_{\gD^{(k)}}\left(h^{*(k)}\right)\right] + \overline{d}_{\gH}(\gD_S^{fed}) + d_{\gG,\gF}(\gD_S^{fed}, \gD_T),
    \end{align}
    where $h^{*(k)}$ denotes the optimal local model of client $k$ (see Section ~\ref{sec:prelim}).
\end{theorem}


\textbf{Semantic Interpretation}. In Theorem \ref{thrm:gen-ub} , the optimal target loss $L^*_{tgt}$
is bounded by the sum of three terms on the RHS of Equation~\ref{eq:gen-bound}: (1) the averaged optimal local loss $\frac{1}{K}\sum_{k=1}^K\left[L_{\gD^{(k)}}\left(h^{*(k)}\right)\right]$; (2) the intra-discrepancy of the source domain $\overline{d}_{\gH}(\gD_S^{fed})$; (3) the discrepancy  between the source and target domains, $d_{\gG,\gF}(\gD_S^{fed}, \gD_T)$. Therefore, the bound on the optimal target loss can be tightened by making the optimal source loss smaller, and by lowering the intra-source  and source-target discrepancy. The later two terms can be controlled through regularization as detailed below.

\textbf{Tightening the Bound via Regularization over $\gF$.} Given two feature extractor function classes $\gF_1$ and $\gF_2$, if $\gF_1 \subset \gF_2$, we have $\gH_1 \subset \gH_2$ where $\gH_1 = \gG \times \gF_1$ and $\gH_2 = \gG \times \gF_2$.
From Definition~\ref{def:cc-div} and~\ref{def:fed_GF_dvg}, it is straightforward to see that $ \overline{d}_{\gH_1}(\gD_S^{fed}) \leq \overline{d}_{\gH_2}(\gD_S^{fed})$ and $d_{\gG, \gF_1}(\gD_S^{fed}, \gD_T) \leq d_{\gG, \gF_2}(\gD_S^{fed}, \gD_T)$. This indicates that any general regularizer on $\gF$ can lead to a decrease in the latter two terms of the RHS of Theorem~\ref{thrm:gen-ub}. However, shrinking the expressive power of $\gF$ will inevitably increase the optimal source loss, which is the first term on the RHS of the expression in Theorem~\ref{thrm:gen-ub}. Therefore, using general regularization, one has to trade-off the \textit{optimal source loss}, \textit{TFL-specific cross-client discrepancy}, and \textit{TFL-specific source-target discrepancy}.

Based on Theorem~\ref{thrm:gen-ub} and the follow-up discussion, we aim to answer the question: how should one design a \textit{practical regularizer} that can inherently tighten our bound on TFL? We give an  answer to this question in Section~\ref{subsec:bound-cross-client}. There, Theorem~\ref{thrm:ub-lr} shows that a transferability-boosting pretraining regularization should decrease the \textit{cross-client Jacobian variance} while at the same time increase the \textit{cross-client averaged Jacobian norm}.
\subsection{Practical Bounds Based on Cross-Client Statistics}
\label{subsec:bound-cross-client}

The goal of FL is to estimate the optimal solution $f^*$ by updating the global model through multiple federated rounds
\footnote{One federated round comprises local model initialization, local training, local model transmission, global model aggregation, and global model broadcasting}.
Denote the total number of rounds by $P$ and the output global model after round $P$ by $f_P$. Denote the target loss under $f_P$ instead of $f^*$ as $\widehat{L}^*_{tgt}$ (which is an estimate of $L^*_{tgt}$). At the end of local training during any round $p\leq P$, let $h^{(k)}_{p}$ with weight $w_p^{(k)}$ represent the model of client $k$, and let $h_{p}$ with weight $w_p$ represent the global model.
Denote the 
Jacobian (gradient) of the loss $l$ w.r.t the model weights at domain $k$ as $J^{(k)}(w)=\sE_{\gD_S^{(k)}} \nabla_{w_h} l(h(x),y)|_{w_h = w}$. 

Throughout the remainder of the manuscript, we use $J_p^{(k)}$ as shorthand for $J^{(k)}\left(w_{p}\right)$. Furthermore, we denote the learning rate of agent $k$ at round $p$ by $\lambda_p^{(k)}$. We make a single-step local update assumption (Assumption~\ref{asp:one-gd}) and use the definitions for cross-client statistics (Definition~\ref{def:cc-stat}) to derive bounds exploiting the cross-client statistics from Lemma~\ref{lem:ub-tgtloss} and Theorem~\ref{thrm:ub-lr}).

\begin{assumption}[Single-Step Local Update]
    \label{asp:one-gd}
    During local training, all clients perform one step of gradient descent (GD) to update their model for transmission,
    $
        w_{p+1}^{(k)} = w_{p} - \lambda^{(k)}_{p+1}\cdot J^{(k)}(w_p)
    $. This is a common assumption in the FL literature~\cite{mcmahan2017communication,ma2023local}. 
    \footnote{\blue{We briefly discuss the Stochastic Gradient Descent approach in Appendix~\ref{sec:apdx-sgd}.}}
\end{assumption}

\begin{definition}[Cross-Client Statistics]
    \label{def:cc-stat}
    At federated round $p$, given $K$ clients with local Jacobians $\{J_p^{(k)}\}_{k\in[K]}$, we define the \textit{cross-client averaged Jacobian norm} $\|J_p\|_2$ and the  \textit{cross-client Jacobian variance}, respectively, as
    \begin{align}
        \|J_p\|_2 = \left\|\frac{1}{K}\sum_{k} J^{(k)}_p\right\|_2,\ \text{ and }\
        \sigma^2_{p} = \frac{1}{K}\sum_{k}
        \left\|J^{(k)}_p\right\|_2^2 - \left\|\frac{1}{K}\sum_{k} J^{(k)}_p\right\|_2^2.
    \end{align}
\end{definition}
Note that we assumed the loss function to have a $\alpha$-Lipschitz continuous gradient. When the gradient is large, $\alpha$ is also large, and to make $\beta_1(\lambda)\geq 0$, $\lambda$ has to be close to $0$. In this case, the absolute value of the second term can be small and that of the third term can be large.
\begin{lemma}[Loss Bound Using Cross-Client Statistics]
    \label{lem:ub-tgtloss}
    Under Assumptions~\ref{asp:one-gd} and~\ref{asp:conv-smooth}, and the cross-client statistics defined in Definition ~\ref{def:cc-stat}, after $P$ rounds of federated pretraining one has
    \begin{align}
        \label{eq:Mrd-tdf-fed-ub}
        \widehat{L}^*_{tgt}\leq L_{src}\left(h_{0}\right)
        -\sum_{p=0}^{P-1}
        \beta_1({\lambda}_{p+1})
        \|J_p\|_2^2 + \sum_{p=0}^{P-1}
        \beta_2({\lambda}_{p+1})
        \sigma^2_{p} + d_{\gG, \gF}(\gD_S^{fed},\gD_T),
    \end{align}
    where $h_{0}$ is the initial global model and $\beta_2(\lambda) = \frac{\alpha{\lambda}^2}{2},\ \beta_1(\lambda) = {\lambda} - \beta_2(\lambda)$.
\end{lemma}
Note that \emph{during the training process}, a large Jacobian norm can promote transferability (the Jacobian is expected to be small at the end of training).

\textbf{Interpretation of Key Terms}. Lemma ~\ref{lem:ub-tgtloss} shows that the \textit{target loss} of the finetuned pretrained model (LHS) is bounded by a sum (RHS) involving four key terms:
1) $L_{src}(h_0)$, the \textit{initial source loss};
2) $\|J_p\|_2$, the \textit{cross-client averaged Jacobian norm};
3) $\sigma_p^2$, the \textit{cross-client Jacobian variance}; 4) $d_{\gG, \gF}(\gD_S^{fed},\gD_T)$, the \textit{TFL-specific source-target domain divergence}. Since $L_{src}(h_0)$ is fixed throughout the pretraining process, and $d_{\gG, \gF}(\gD_S^{fed},\gD_T)$ can be reduced using general regularization, we focus on analyzing the two tunable cross-client statistics, $\|J_p\|_2$ and $\sigma_p^2$.

\textbf{Influence of the Cross-Client Statistics on the Bound}. We note that during pretraining, both the coefficients $\beta_1(\lambda_{p+1})$ and $\beta_2(\lambda_{p+1})$ have to be positive (see how to ensure these  constraints in Appendix ~\ref{apdx:pstv-coef}). Therefore, a larger $\|J_p\|_2$ and a smaller $\sigma^2_p$ tighten the upper bound, indicate a lower target loss, and thus suggest better model transferability.

\textbf{Coefficients Quadratic w.r.t. the Learning Rates}. Lemma~\ref{lem:ub-tgtloss} involves coefficients $\beta_1(\lambda)$ and $\beta_2(\lambda)$ that are quadratic in the learning rate $\lambda$; thus, we can further tighten the bound by optimizing the learning rates across different rounds (Assumption~\ref{asp:opt-lr}). The tightened bound is given in Theorem~\ref{thrm:ub-lr}.

\begin{assumption}[Optimal Learning Rates Across Rounds]
    \label{asp:opt-lr}
    In each round $p$ ($2\leq p\leq P$), we use an optimal learning rate for local training
    $
        \lambda_{p}^* = \frac{K\cdot\|J_{p-1}\|_2^2}{\alpha \cdot \sum_k\|J_{p-1}^{(k)}\|_2^2}.
    $
    A similar analysis and assumptions on optimal learning rate has also been used in~\cite{haddadpour2019convergence,deng2020adaptive,cui2022optimal,wu2023faster,ma2023local}.
\end{assumption}


\begin{theorem}[Tightened Bound Based on Cross-Client Statistics]
    \label{thrm:ub-lr}
    \blue{By optimizing the bound in Lemma \ref{lem:ub-tgtloss} with respect to the learning rates $\lambda$ as governed by Assumption~\ref{asp:opt-lr},} and under Assumption~\ref{asp:conv-smooth} (Convexity and Smoothness), the estimated \textbf{optimal} target loss is bounded as follows:
    \begin{align*}
        \widehat{L}_{tgt}^* \leq L_{src}\left(h_{0}\right)- \sum_{p=0}^{P-1}
        \frac{2\|J_p\|_2^2}{\alpha(1+ {\sigma^2_p}{\|J_p\|_2^{-2}})} + d_{\gG, \gF}(\gD_{S}^{fed}, \gD_T).
    \end{align*}
\end{theorem}

\textbf{Semantic Interpretation}. Theorem~\ref{thrm:ub-lr} indicates that by optimizing the learning rates at each round, the bound for $\widehat{L}^*_{tgt}$ can be made smaller through 1) a smaller source loss $L_{src}(h_0)$;
2) a larger cross-client average Jacobian norm $\|J_p\|_2$; 3) a smaller cross-client Jacobian variance $\sigma_p^2$, and 4) a smaller source-target domain divergence $d_{\gG, \gF}(\gD_S^{fed}, \gD_T)$. These are consistent with Lemma~\ref{lem:ub-tgtloss}.


\blue{\textbf{Intuitive Explanation}. 
\textbf{(a)} \textit{Increasing the Cross-Client Averaged Jacobian Norm $\|J_p\|$}. 
In early training,  it is crucial to avoid a small \(J_p\), as it can ``trap'' clients in local minima or cause overfitting of local distributions. Note that increasing \(J_p\) will not induce an excessively large second term (i.e., a term that \(\to \infty\)) in Theorem \ref{thrm:ub-lr}, as increments are constrained through the Lipschitz condition \(\|J_1 - J_2\| \leq \alpha \|x_1 - x_2\|\). In the final stages of training, \(\|J_p\|\) naturally decreases due to convergence.
\textbf{(b)} \textit{Decreasing the Cross-Client Jacobian Variance $\sigma_p$}. A small \(\sigma_p\) promotes domain similarity, preventing local overfitting and enhancing global model transferability.
\textbf{(c)} \textit{Trade-Off Between Increasing \( \|J_p\| \) and Decreasing \( \sigma_p \)}. A larger \(\|J_p\|\) induces more substantial local updates but can increase variance \(\sigma_{p+1}\). Thus, Theorem 2 underscores balancing this trade-off: enlarging \( \|J_p\| \) while maintaining a small \(\sigma_p\).}

\textbf{\blue{Challenges in} Regularization of the Jacobians}. Certain regularization of the Jacobians  can induce local Jacobian alignment and reduce $\sigma_p^2$ (e.g., the regularization of the alignment between local gradients and the global gradient proposed in FedIIR ~\cite{guo2023out}). However, local regularization can naturally \textit{impede} the growth of local Jacobian norms $\{J_p^{(k)}\}_{k\in[K]}$ and subsequently prevent  $\|J_p\|_2$ from increasing. Thus, \blue{such prior regularization} may not necessarily improve transferability. As an example, consider the extreme case where
$\sigma^2_{p} = 0$. Then, the bound in Theorem~\ref{thrm:ub-lr} takes the form $
    \widehat{L}_{tgt}^* \leq L_{src}\left(h_{0}\right)- \sum_{p=0}^{P-1}
    \frac{2}{\alpha}\|J_p\|_2^2 + d_{\gG, \gF}(\gD_{S}^{fed}, \gD_T)
$, where a small $\|J_p\|_2$ can severely degrade the bound. This indicates that boosting transferability by \blue{prior regularization, such as} controlling $\sigma_p^2$ itself, has limitations, and tuning $\|J_p\|$ during the pretraining stage is of crucial importance. We provide a solution to this problem in Section~\ref{sec:method}.

\section{Our Algorithm}
\label{sec:method}
\vspace{-0.1in}
\textbf{\blue{Algorithmic Solution for the Theoretical Challenges}}. From Lemma~\ref{lem:ub-tgtloss} and Theorem ~\ref{thrm:ub-lr}, we can see that certain round-wise FL-specific statistics, the \textit{cross-client averaged Jacobian norm} $\|J_p\|_2$ and \textit{cross-client Jacobian variance} $\sigma_p^2$ control the bound on the target loss. The challenge is that, while $\sigma_p^2$ can be reduced using straightforward techniques (i.e., such as gradient alignment from FedIIR~\cite{guo2023out}), such techniques unavoidably prevent $J_p$ from increasing properly. We therefore propose our FedGTST approach which reduces $\sigma_p^2$ while enlarging $\|J_p\|_2$. A round-wise description of FedGTST is given in Algorithm~\ref{alg:Fed-GTST}.

\textbf{Tuning the Cross-Client Jacobian Variance $\sigma_p^2$.} 
The cross-client Jacobian variance is controlled via regularization at the local client level (Line \ref{line:lc-almt} of Algorithm~\ref{alg:Fed-GTST}). The local clients, upon receiving a \textit{guide norm} $\gamma$ from the server (explanation deferred to subsequent sections), implement a regularized local training protocol. While client $k$ during standard local training uses the objective $L_{\gD^{(k)}}(h)$, during regularized local training he/she/they use
\begin{align}
    \label{eq:reg-obj}
    L_{\gD_S^{(k)}}\left(h\right) + \xi\cdot\left(\|J^{(k)}(h)\|_2 - \gamma\right)^2
\end{align}
instead, where $\xi$ is the penalty coefficient for the regularization term. This type of  regularization intuitively aligns each local Jacobian norm $\|J^{(k)}\|$ with the guide norm $\gamma$, which results in a $\sigma^2_J$ that is smaller than that of standard FL (i.e., FedAVG).
\vspace{-0.2cm}
\begin{algorithm}[H]
    \label{alg:Fed-GTST}
    \caption{FedGTST (Round $p$)}
    \begin{algorithmic}[1]
        \State Randomly select a set $\phi\subset[K]$ for regular training.
        \label{line:vnl-slct}
        \While{$k \leq K$, Client $k$ should}
        \State Receive the \emph{guide} norm $\gamma_{p-1}$ and global model $h_{p-1}$ from the server.
        \State Initialize the local model to $h_{p-1}$.
        \State Update the local model to $h_p^{(k)}$ by training with
        $L_{\gD_S^{(k)}}\left(h\right) + \xi\cdot\left(\|J^{(k)}(h)\|_2 - \gamma\right)^2$.
        \label{line:lc-almt}
        \State Calculate the surrogate Jacobian norm $\gamma^{(k)}_p = \gamma^{reg,(k)}_p= \|J_p^{(k)}\|_2$.
        \label{line:surr-stt}
        \If{$k\in \phi$}
        \State Train with $L_{\gD^{(k)}}(h)$ to obtain Jacobian norm $\gamma_p^{vnl,(k)}$ .
        \State Update the surrogate Jacobian norm $\gamma_p^{(k)} = \max\left(\gamma^{reg,(k)}_p, \gamma^{std,(k)}_p\right)$.
        \EndIf
        \label{line:surr-end}
        \State Transmit the model $h_{p}^{(k)}$ and norm $\gamma_p^{(k)}$ to the server.
        \EndWhile
        \State The server aggregates the client models into a global model $h_p$ and sets the guide norm to $\gamma_p = \max_k(\gamma_p^{(k)})$.
        \label{line:gdnm-stt}
        \State The server broadcasts $h_p$ and $\gamma_p$ to all clients.
        \label{line:gdnm-end}
    \end{algorithmic}
\end{algorithm}

\vspace{-0.5cm}
\textbf{Tuning the Cross-Client Averaged Jacobian Norm $\|J_p\|_2$.}
 The cross-client averaged Jacobian norm can be increased via an exchange protocol that includes: 1) clients calculating surrogate norms for transmission; 2) the server computing and broadcasting a guide norm; 3) clients performing local alignment using the guide norm.

\textbf{Surrogate Norms}. 
A decrease in the local Jacobian norm $\|J^{(k)}\|$ prevents the cross-client averaged Jacobian norm $\|J\|$ from growing. We mitigate such a decrease by forcing a small portion of the clients to implement both regularized training and standard training (Line \ref{line:vnl-slct}), to generate a pair of Jacobian norms $\gamma^{reg,(k)}$ and $\gamma^{std,(k)}$. The client than chooses the larger norm, $\gamma^{(k)} = max(\gamma^{reg,(k)}, \gamma^{std,(k)})$ as a surrogate norm for transmission (Lines~\ref{line:surr-stt} to~\ref{line:surr-end}).

\textbf{Server Guide Norm}. The clients send their local Jacobian norms $\gamma^{(k)}:= \|J^{(k)}\|_2$ to the server, and the server broadcasts the largest norm received as its guide norm $\gamma := \max_k \gamma^{(k)}$ (Line \ref{line:gdnm-stt}). 

\textbf{Local alignment}. The local regularizer from Equation~(\ref{eq:reg-obj}) reduces the variance, but  can also force an increase of the averaged Jacobian norm. It forces all local Jacobian norms to align with the guide norm $\gamma$ (Line \ref{line:lc-almt}). Since $\gamma$ has been boosted both at the clients' and server levels, the alignment leads to larger local Jacobian norms.

\textbf{Communication Cost.}
Most TFL methods, as already described, require large communication overheads between the clients and the server. For example, FedIIR requires exchanging Jacobians, which have the same dimension as the model weights; FedGTST only requires exchanging norms.
\vspace{-0.1in}

\section{Experiments}
\label{sec:exp}
\subsection{Experimental Setting}
\textbf{Transfer tasks.}
\blue{We investigate three transfer tasks utilizing fully-annotated data}:  a) MNIST~\cite{deng2012mnist} to MNIST-M~\cite{ganin2016domain}, b) CIFAR-10~\cite{krizhevsky2009learning} to SVHN~\cite{netzer2011reading}, and \blue{c) cross-domain transfer in DomainNet~\cite{peng2019moment}.} All transfer tasks have been used as benchmarks in existing TL research~\cite{ bousmalis2016domain,damodaran2018deepjdot,ganin2016domain,  
lee2021dranet,
long2015learning,xu2022adversarially} (also, see Appendix~\ref{apdx:ds-intro}). \blue{Moreover, DomainNet is a standard large-scale dataset for TFL  studies, also used by FedSR.}

\textbf{\blue{Non-iid Distributed Source (Local) Domains}}. \blue{The pretraining phase is conducted via FL on source (local) domains}. Marginal distribution shift is an important phenomenon in FL~\cite{li2019convergence} since the access to \blue{classes (or categories)} is not the same for all participating entities. However, some federated datasets tested by existing transferable FL methods do not reflect marginal distribution shifts when constructing source local domains (i.e., in the Rotated-MNIST benchmark in FedSR\cite{nguyen2022fedsr}, all clients have access to all \blue{classes}). To address this issue,
\blue{we employ the following methods for constructing non-iid local domains:}
\vspace{-0.2cm}
\begin{itemize}
    \item \blue{For MNIST or CIFAR-10, unless otherwise specified,} we follow the  approach in~\cite{li2019convergence} and let each client have access to \blue{only a subset of the categories}. The \blue{category} selection rule and data sampling method is described in Appendix \ref{apdx:FL-noniid}.
    \blue{Besides, we run additional experiments for Dirichlet sampling, a method commonly used in FL~\cite{LiHS21,Li2021MOON, luo2021no, acar2021federated, tang2022virtual, zhang2022dense, lai2022fedscale, TangZS0L0024}.}
    \item \blue{For DomainNet, which compromises six distinct domains, we follow the leave-one-out strategy used in FedSR: one domain is designated as the target domain, while the remaining five domains serve as source domains, with each assigned to an individual client.}
\end{itemize}

\vspace{-0.2cm}
\textbf{Federated System Size}. Large system sizes are inherent to FL systems, where the number of clients can be as large as $100$~\cite{Li_2021_CVPR,li2019convergence}. In such a case, it is more challenging for the global model to achieve good performance. However, existing TFL methods typically use a very small number of domains ($\leq 5$ for FedSR and StableFDG). In contrast, we allow the system size to cover a broad range of values, including $10$ (small), $50$ (medium), and $100$ (large) clients.

\textbf{Evaluation matrix and backbones}. We measure domain transferability via $acc_{tgt}$, the accuracy of a pretrained model finetuned on the target domain. For backbone selection, we use both LeNet~\cite{726791} and ResNet18\cite{he2016deep} to represent different levels of backbone complexity. 

\textbf{Baselines}. We consider the following TFL algorithms as our main baselines: FedAVG, FedSR, and FedIIR. \blue{Furthermore, we also compare our findings to Scaffold~\cite{karimireddy2020scaffold}, an advanced FL approach used to address data heterogeneity.} We do not report results for FedADG, FedCDG, FedGTST, and StableFDG since they do not ensure privacy and/or underperform compared to the main baseline. 
    

\textbf{\blue{Settings for Pretraining (FL on source local domains)}}. 1) \textit{Local epochs}. To conform with our theoretical assumptions, unless specified otherwise, we set the local epoch of each client to $1$. We also investigate the system performance with $10$ local epochs in Appendix~\ref{apdx:addi-exp}. 2) \textit{Number of participants per round}. We allow $50\%$ of the clients to participate in each round (e.g., if $K = 50$, there are $25$ participants per round). 3) \textit{Number of clients performing standard local training}. To boost the Jacobian norm, a subset of clients conducts standard local training (we set the subset size to $10\%$ of the total number of clients).

\textbf{Additional Settings}. Unless specified otherwise, for local training on the source datasets we use the Adam~\cite{kingma2014adam} optimizer with coefficients $(\beta_1, \beta_2) = (0.9,0.999)$ (note that these $\beta$ values are not to be confused with the coefficients from our theoretical analysis). Our initial learning rate equals $0.01$ and then decays by a factor of $10$ per $50$ rounds, with an early stop trigger  of $10$ rounds. We apply the standard cross-entropy loss. 
We also set the pretraining batch size to $256$ for MNIST $\to$ MNIST-M and to $128$ for CIFAR10 $\to$ SVHN. For both tasks, we use a finetuned learning rate $0.005,$ with weight decay $0.0001$. All results reported in Section~\ref{subsec:exp-res-main} are averaged over three runs.

\subsection{Results}
\label{subsec:exp-res-main}

\begin{table}[ht]
\centering
\begin{tabular}{cccccc}
\hline
\cline{1-6}
& \multicolumn{2}{c}{{ MNIST→ MNIST-M}}& \multicolumn{2}{c}{{ CIFAR10 → SVHN}}&\\ 
\cline{2-3}
\cline{4-5}
\multirow{-2}{*}{Method} & { LeNet} & { ResNet} & { LeNet} & { ResNet} & \multirow{-2}{*}{Average} \\ 
\hline
{FedAVG}   & { 73.8$\pm$0.7}  & { 81.6$\pm$0.2}   & { 64.4$\pm$0.5}  & { 72.0$\pm$1.0}   & 73.0\\
{FedSR}    & { 75.0$\pm$0.9}  & { 80.6$\pm$0.1}   & { 65.9$\pm$0.6}  & { 71.3$\pm$0.4}   & 73.2\\
{FedIIR}   & { 74.5$\pm$0.3}  & { \textbf{82.7$\pm$0.7}}   & { 66.2$\pm$1.0}  & { 73.8$\pm$0.2}   & 74.3\\
{Fed-GTST} & { \textbf{76.2$\pm$0.9}}  & { 82.3$\pm$0.5}   & \textbf{ 70.1$\pm$0.8}  & \textbf{ 74.5$\pm$0.3}   & \textbf{75.8}\\ 
\hline
\cline{1-6}
\end{tabular}
\vspace{0.04in}
\caption{Target accuracy (\%) of the finetuned model pretrained on a small number of clients ($K = 10$). FedGTST outperforms other methods across both tasks and both backbones; for the example MNIST to MNIST-M with a LeNet backbone, FedGTST outperforms FedIIR and FedSR by around $4\%$.}
\label{tab:res-main-small}
\end{table}
\textbf{Transferability results}. FedGTST exhibits \textbf{significantly improved transfer performance} when compared to baselines across a range of tasks, system sizes, and backbone architectures. For the example of CIFAR10$\to$SVHN with $100$ clients and a LeNet backbone, FedGTST outperforms FedIIR by $7.6\%$ and  FedSR by $9.8\%$. The results for a small, medium and large federated system ($K = 10$, $K=50$ and $K=100$) are reported in Tables~\ref{tab:res-main-small},~\ref{tab:res-main-mid} and ~\ref{tab:res-main-large}, respectively. 


\begin{table}[ht]
\centering
\begin{tabular}{cccccc}
\hline
\cline{1-6}
& \multicolumn{2}{c}{{ MNIST→ MNIST-M}} & \multicolumn{2}{c}{{ CIFAR10 → SVHN}}&\\ 
\cline{2-5}
\multirow{-2}{*}{Method}& { LeNet}& { ResNet}& { LeNet}& { ResNet} & \multirow{-2}{*}{Average} \\ 
\hline
{ FedAVG}   & { 63.5 $\pm$ 0.3}& { 72.4$\pm$0.2}& { 59.1$\pm$0.5}& { 65.5$\pm$0.1}& 65.1\\
{ FedSR}    & { 64.0$\pm$0.6}& { 73.1$\pm$0.3}& { 58.9$\pm$0.1}& { 66.8$\pm$0.5}& 65.7\\
{ FedIIR}   & { 64.7$\pm$0.2}& { 73.9$\pm$0.4}& { 59.7$\pm$0.3}& { 66.2$\pm$0.1}& 66.1\\
{ Fed-GTST} & { \textbf{69.0$\pm$0.8}} & { \textbf{79.2$\pm$0.4}} & { \textbf{63.5$\pm$0.1}} & { \textbf{71.1$\pm$0.2}} & \textbf{70.7}\\ 
\hline
\cline{1-6}
\end{tabular}
\vspace{0.04in}
\caption{Target accuracy (\%) of the finetuned model pretrained on a medium number of clients ($K = 50$). FedGSTS outperforms other methods across both tasks and both backbones; for both the examples of MNIST$\to$MNIST-M and CIFAR10$\to$SVHN with a ResNet18 backbone, FedGTST outperforms FedIIR and FedSR by around $5\%$.}
\label{tab:res-main-mid}
\end{table}
\vspace{-0.5cm}
\begin{table}[ht]
\centering
\begin{tabular}{cccccc}
\hline
\cline{1-6} 
& \multicolumn{2}{c}{{ MNIST→ MNIST-M}} & \multicolumn{2}{c}{{ CIFAR10 → SVHN}}&\\ 
\cline{2-5}
\multirow{-2}{*}{Method}& { LeNet}& { ResNet}& { LeNet}& { ResNet}& \multirow{-2}{*}{Average} \\ 
\hline
{ FedAVG}& { 48.7}$\pm$0.1 & { 61.1}$\pm$0.3& { 41.2}$\pm$0.2& { 51.7}$\pm0.5$ & 50.7 \\
{ FedSR}& { 49.2}$\pm$0.2 & { 59.8}$\pm$0.3 & { 42.6}$\pm$0.1 & { 52.0}$\pm$0.3 & 50.9 \\
{ FedIIR} & { 51.9}$\pm$0.4& { 61.3}$\pm$0.1& { 44.8}$\pm$0.4& { 55.8}$\pm$0.2& 53.4\\
{ Fed-GTST} & { \textbf{57.5}$\pm$0.3} & { \textbf{67.6}$\pm$0.2} & { \textbf{52.4}$\pm$0.1} & { \textbf{63.1}$\pm$0.2} & \textbf{60.2} \\ 
\hline
\cline{1-6}
\end{tabular}
\vspace{0.04in}
\caption{Target accuracy (\%) of the finetuned model pretrained on a large number of clients ($K = 100$). FedGTST outperforms other methods across both tasks and for both backbones; for CIFAR10$\to$SVHN with a ResNet18 and LeNet backbone, FedGTST outperforms FedIIR and FedSR by $~7\%$.}
\label{tab:res-main-large}
\end{table}

\vspace{-0.4cm}
\textbf{Discussion.} Besides the significant performance \blue{gain} of FedGTST over baseline methods, we also observe that 1) for transfer tasks in which the backbone and method are fixed, transferability generally decreases with the system size. Importantly, FedGTST improves the baselines more significantly for large system sizes; 2) when the system size, transfer task and the method are fixed, the more ''complex'' the backbone (ResNet18 vs LeNet), the better the transferability.


\textbf{Additional results}. We defer reporting and discussing additional results in Appendix \ref{apdx:addi-exp}, which include (a) an investigation on \textit{hyper-parameter sensitivity}
the \textit{convergence speed}, and \textit{cross-client statistics}, (b) \blue{results regarding \textit{DomainNet}, \textit{Dirichlet Sampling}, and \textit{Scaffold}.}


\vspace{-0.1in}
\section{Limitations}
\vspace{-0.1in}
\textbf{Increased Local Computational Cost}. While FedGTST only induces a negligible communication overhead among clients and the server, we note that at the level of a \emph{small subset of clients} we need to conduct both regularized training and standard training to boost $\|J_p\|_2$. This increases the local computational burden but to a very small extent. Nevertheless, in future works, we will explore algorithms with lower local computational costs. 

\textbf{Potentially Loose Bounds}. Although existing studies have added to our understanding of transferable federated learning, our work is the first to derive a bound on a direct measure of transferability (the target loss). We believe that the bound can be tightened.


\vspace{-0.3cm}
\section{\blue{Conclusion}}
\vspace{-0.4cm}
\blue{We introduced FedGTST, a federated learning algorithm aimed at enhancing global transferability. Inspired by theoretical insights, FedGTST integrates cross-client information on averaged Jacobian norms and Jacobian variance. Our work addresses key challenges in existing methods, such as privacy violations and an overemphasis on local transferability. Experimental results demonstrate significant performance improvements over baseline models.}

\vspace{-0.3cm}
\section*{\blue{Acknowledgement}}
\vspace{-0.3cm}
\blue{This work was supported by the Air Force Office of Scientific Research under award number FA9550-23-1-0107, the NSF CAREER Award under Grant No. EPCN-1944403 as well as the NSF Awards CCF 1956384 and 2402815, and SVCF CZI 2022-249120.}

\newpage
\bibliographystyle{plainnat}
\bibliography{ref}

\newpage
\appendix

\section{Additional preliminaries}
\begin{definition}[$\gH$-discrepancy~\cite{cortes2019adaptation,zhao2019learning}]
    \label{def:h_divergence}
    Given a model function class $\gH$ and two data distributions $\gD_S$, $\gD_T$, the $\gH$-discrepancy between $\gD_S, \gD_T$ is defined as
    \begin{align}
        \label{eq:h_divergence}
        d_{\gH}(\gD_S,\gD_T):= \sup_{h\in \gH}|L_{\gD_S}(h) -L_{\gD_T}(h) |.
    \end{align}
\end{definition}
{\remark This type of discrepancy, as well as the related $\gH$-divergence, have been frequently used in the analysis of
domain adaptation~\cite{CORTES2014103,zhao2019learning,mansour2009domain}, and we follow this trend.} 

\section{Proof of our bounds}\label{app:lem:round-ub}
\begin{lemma}[Theorem 2.5 in \cite{xu2022adversarially}]
\label{lem:gen-tsf-ub}
Suppose we perform pretraining on the source domain $\gD_S$ to obtain $f$ and $g$ at the server, and fine-tune the model on the target domain $\gD_T$ to obtain a new classifier $g_T$. We then have
\begin{align}
\label{eq:gen-tsf-ub}
    L_{\gD_T}(g^*_T \circ f^*) \leq L_{src} (g^*  \circ f^* ) + d_{\gG,\gF} (\gD_S,\gD_T).
\end{align}
\end{lemma}

Equation~(\ref{eq:gen-tsf-ub}) shows that the transfer loss can be upper bounded by the sum of the loss on the source domain and the divergence between two domains.

\begin{lemma}
\label{lem:gnr-fed-ub}
With Assumption \ref{asp:conv-smooth} (Convexity and Smoothness) and Definition \ref{def:cc-div} (Cross-Client Divergence), we have 
\begin{align}
\label{eq:gnr-fed-ub}
    L_{src}(h^*)\leq \frac{1}{K}\sum_{k=1}^K\left[L_{\gD^{(k)}}\left(h^{*(k)}\right)\right] + \overline{d}_{\gH}(\gD_S^{fed}).
\end{align}
\end{lemma}

\begin{proof}
The loss function $l$ is assumed to be convex w.r.t. the parameters of the model $h$, where the parameters are denoted by $w_{h}$. Based on the aggregation rule in Section \ref{sec:prelim}, we have $w_{h}=\frac{1}{K}\sum_{k}w_{h}^{(k)}$, where $w_{h}$ and $w_{h}^{(k)}$ stands for the weights for global model $h$ and that for local model $h^{(k)}$, respectively. We therefore have
\begin{align*}
    L_{src}(h^*) &= \fK\sum_{k_1} L_{\gD^{(k_1)}}\left(h^*\right)\\
    & = \fK\sum_{k_1} \sE_{\gD^{(k_1)}}\;l\left(h\left(x,w^*_{h}\right),y\right)\\
    & = \fK\sum_{k_1} \sE_{\gD^{(k_1)}}\;l\left(h\left(x,\fK\sum_{k_2}w^{*(k_2)}_{h}\right),y\right)\\
    &\stackrel{(a)}{\leq} \frac{1}{K^2}\sum_{k_1,k_2} \sE_{\gD^{(k_1)}}\;l\left(h\left(x,w_{h}^{*(k_2)}\right),y\right)\\
    & = \frac{1}{K^2}\sum_{k_1,k_2} L_{\gD^{(k_1)}}\left(h^{*(k_2)}\right)\\
    & \stackrel{(b)}{\leq} \frac{1}{K^2}\sum_{k_1,k_2} \left[L_{\gD^{(k_1)}}\left(h^{*(k_1)}\right) + d_{\gH}\left(\gD^{(k_1)}, \gD^{(k_2)}\right)\right]\\
    & = \fK\sum_{k_1}L_{\gD^{(k_1)}}\left(h^{*(k_1)}\right) + \frac{1}{K^2}\sum_{k_1,k_2}d_{\gH}\left(\gD^{(k_1)}, \gD^{(k_2)}\right)\\
    & \leq \fK\sum_{k_1}L_{\gD^{(k_1)}}\left(h^{*(k_1)}\right)+\overline{d}_{\gH}(\gD_S),
\end{align*}
where $\frac{1}{K^2}\sum_{k_1,k_2}d_{\gH}\left(\gD^{(k_1)}, \gD^{(k_2)}\right)\leq \frac{1}{K(K-1)}\sum_{k_1\neq k_2}d_{\gH}\left(\gD^{(k_1)}, \gD^{(k_2)}\right) =\overline{d}_{\gH}(\gD_S)$ as in Definition \ref{def:cc-div}. Here $(a)$ follows from the convexity assumption for $l$ w.r.t the parameters, while $(b)$ follows from Lemma~\ref{lem:gen-tsf-ub} based on the argument below:
\begin{align*}
L_{\gD^{(k_1)}}\left(h^{*(k_2)}\right) &= L_{\gD^{(k_1)}}\left(h^{*(k_2)}\right) - L_{\gD^{(k_2)}}\left(h^{*(k_2)}\right) + L_{\gD^{(k_2)}}\left(h^{*(k_2)}\right) \\
&\leq \left|L_{\gD^{(k_1)}}\left(h^{*(k_2)}\right) - L_{\gD^{(k_2)}}\left(h^{*(k_2)}\right)\right| + L_{\gD^{(k_2)}}\left(h^{*(k_2)}\right) \\
&\leq d_{\gH}\left(\gD^{(k_1)}, \gD^{(k_2)}\right) + L_{\gD^{(k_2)}}\left(h^{*(k_2)}\right).
\end{align*}
\end{proof}

\begin{theorem}[Theorem \ref{thrm:gen-ub}]
    Under Assumptions~\ref{asp:conv-smooth} (Convexity and Smoothness), the optimal target loss is bounded by
    \begin{align}
        \label{eq:gen-bound-rstt}
        L^*_{tgt} \leq \frac{1}{K}\sum_{k=1}^K\left[L_{\gD^{(k)}}\left(h^{*(k)}\right)\right] + \overline{d}_{\gH}(\gD_S^{fed}) + d_{\gG,\gF}(\gD_S^{fed}, \gD_T),
    \end{align}
    where $h^{*(k)}$ denotes the optimal local model of client $k$ (see Section ~\ref{sec:prelim}).
\end{theorem}

\begin{proof} 
Since
\begin{align*}
   L_{tgt}^* =  L_{\gD_T}\left(g_T^* \circ f^{*}\right) \leq  L_{\gD^{(k)}}\left(g^* \circ f^{*}\right) + d_{\gG, \gF}(\gD^{(k)}, \gD_T),
\end{align*}

we have
\begin{align*}
     L^*_{tgt} &  \leq \frac{1}{K}\sum_{k}\left[L_{\gD^{(k)}}\left(g^* \circ f^{*}\right) + d_{\gG, \gF}(\gD^{(k)}, \gD_T)\right]\\
     &=  L_{src}(h^*) +  d_{\gG,\gF}(\gD_S^{fed}, \gD_T).
\end{align*}
Using Lemma~\ref{lem:gnr-fed-ub}, we have
\begin{align*}
     L^*_{tgt} &  \leq L_{src}(h^*) +  d_{\gG,\gF}(\gD_S^{fed}, \gD_T)\\
     & \leq \frac{1}{K}\sum_{k=1}^K\left[L_{\gD^{(k)}}\left(h^{*(k)}\right)\right] + \overline{d}_{\gH}(\gD_S^{fed})+  d_{\gG,\gF}(\gD_S^{fed}, \gD_T).
\end{align*}

\end{proof}

\begin{lemma}[Bound on round-wise source loss]
\label{lem:round-ub}
Suppose the learning rates of all clients at round $p$ are the same, i.e., $\lambda^{(k)}_p = {\lambda}_{p},\forall k\in[K], p\in[P]$. Under Assumptions~\ref{asp:one-gd} and~\ref{asp:conv-smooth}, we have that
\begin{align}
\label{eq:round-ub}
L_{src}\left(h^{}_{p+1}\right)\leq L_{src}\left(h^{}_{p}\right)
    - \beta_1({\lambda}_{p+1}) \|J_p\|_2^2 + 
     \beta_2({\lambda}_{p+1}) \sigma^2_{p}
\end{align}
where 
$
    J_p = \frac{1}{K}\sum_{k} J^{(k)}_p,\ 
    \sigma^2_{p} = \frac{1}{K}\sum_{k} 
    \left\|J^{(k)}_p\right\|_2^2 - \left\|\frac{1}{K}\sum_{k} J^{(k)}_p\right\|_2^2,
$
and $\beta_1(\lambda) = {\lambda} - \beta_2(\lambda),\ \beta_2(\lambda) = \frac{\alpha{\lambda}^2}{2}$.
\end{lemma}

\begin{proof}
Following the same proof idea as for Lemma~\ref{lem:gnr-fed-ub}, we have
\begin{align}
\label{eq:local_loss_0}
L_{src}\left(h_{p+1}\right) \leq \frac{1}{K^2}\sum_{k_1,k_2}L_{\gD^{(k_1)}}\left(h_{p+1}^{(k_2)}\right)
\end{align}
By definition, we also have
\begin{align}
\label{eq:local_loss_1}
L_{\gD^{(k_1)}}\left(h_{p+1}^{(k_2)}\right)=\sE_{\gD^{(k_1)}}l\left(h\left(x,w_{p+1}^{(k_2)}\right),y\right)
\end{align}
From the update rule of GD, we can write
\begin{align}
l\left(h\left(x,w_{p+1}^{(k_2)}\right),y\right)
& = l\left(h\left(x, w_{p} - \lambda_{p+1}\cdot
\left[\nabla_{w_h} L_{\gD^{(k_2)}}\left(h\right)\bigg|_{w_h = w_{p}}\right]\right),y\right) \notag\\
& =  l\left(h\left(x, w_{p} - \lambda_{p+1}\cdot
\left[\nabla_{w} 
\sE_{\gD^{(k_2)}}l\left(h(x_i,w),y_i\right)
\bigg|_{w = w_{p}}\right]\right),y\right) \label{eq:xi_yi}\\
& =  l\left(h\left(x, w_{p} - \lambda_{p+1}\cdot
J^{(k_2)}_p\right),y\right) \label{eq:local_loss_2}
\end{align}
In~(\ref{eq:xi_yi}), $(x_i,y_i)$ are sampled from $\gD^{(k_2)}$, and not to be confused with $(x, y)$ sampled from $\gD^{(k_1)}$. 

Define $\Delta := -\lambda_{p+1}\cdot J^{(k_2)}_p$, we can upper bound the loss in~(\ref{eq:local_loss_2}) by
\begin{align}
\label{eq:local_loss_3}
l\left(h\left(x, w_{p+1}^{(k_2)}\right),y\right)
& = l\left(h\left(x,w_{p} +\Delta\right),y\right)\notag\\
& \leq l\left(h\left(x,w_{p} \right),y\right) +\left[\nabla_w  l\left(h\left(x, w\right),y\right)\big|_{w = w_{p}}\right]^\top \cdot\Delta + \frac{\alpha}{2}\Delta^\top\Delta.
\end{align}
The inequality holds since $l$ is assumed to have $\alpha$-Lipschitz continuous gradient. Combining ~(\ref{eq:local_loss_1}) and~(\ref{eq:local_loss_3}), we have
\begin{align}
\label{eq:local_loss_4}
L_{\gD^{(k_1)}}\left(h_{p+1}^{(k_2)}\right)
&\leq \sE_{\gD^{(k_1)}}
\left[l\left(h\left(x,w_{p} \right),y\right) +\left[\nabla_w  l\left(h\left(x,w \right),y\right)\big|_{w = w_{p}}\right]^\top \cdot\Delta + \frac{\alpha}{2}\Delta^\top\Delta\right]\notag\\
& = L_{\gD^{(k_1)}}\left(h_{p}\right) + \left[J^{(k_1)}_p\right]^\top\cdot\Delta + \frac{\alpha}{2}\Delta^\top\Delta
\end{align}

Next, by ~(\ref{eq:local_loss_0}) and~(\ref{eq:local_loss_4}), we have
\begin{align}
\label{eq:fl_one_round_loss_1}
\quad L_{src}\left(h_{p+1}\right) 
&\leq \frac{1}{K^2}\sum_{k_1,k_2}
L_{\gD^{(k_1)}}\left(h_{p+1}^{(k_2)}\right)\notag\\
& \leq \frac{1}{K^2}\sum_{k_1, k_2}
\left[L_{\gD^{(k_1)}}\left(h_{p}\right) + \left[J^{(k_1)}_p\right]^\top\cdot\Delta + \frac{\alpha}{2}\Delta^\top\Delta\right]\notag\\
& = \frac{1}{K^2}\sum_{k_1, k_2}
\left[L_{\gD^{(k_1)}}\left(h_{p}\right) 
- \lambda_{p+1}\cdot\left[J^{(k_1)}_p\right]^\top\cdot
J^{(k_2)}_p\right. \notag\\
& \left. \quad + 
\frac{\alpha}{2}\cdot(\lambda_{p+1})^2\cdot\left\|
J^{(k_2)}_p\right\|_2^2\right] \notag\\
& = L_{src}\left(h_{p}\right) - \lambda_{p+1}\|J_p\|_2^2 
+ \frac{\alpha}{2K}\cdot\left(\lambda_{p+1}\right)^2 \cdot \sum_{k_2} \left\|
J^{(k_2)}_p\right\|_2^2.
\end{align}
The last equality follows from the definition of $J_p$.

Define $\sigma^2_{p} := \fK \sum_{k=1}^K \left\|
J^{(k)}_p \right\|_2^2 - \left\|J_p\right\|_2^2
$. Then, we have
\begin{align}
\label{eq:fl_one_round_loss_2}
L_{src}\left(h_{p+1}\right) 
& \leq L_{src}\left(h_{p}\right) - \lambda_{p+1}\|J_p\|_2^2 
+ \frac{\alpha}{2} \cdot (\lambda_{p+1})^2 \cdot \left(\sigma^2_{p} + \|J_p\|_2^2 \right) \notag\\
& = L_{src}\left(h_{p}\right) 
- \left(\lambda_{p+1}
- \frac{\alpha\cdot(\lambda_{p+1})^2}{2}\right)\|J_p\|_2^2 
+  \frac{\alpha\cdot(\lambda_{p+1})^2}{2}\sigma^2_{p} \notag\\ 
& = L_{src}\left(h_{p}\right) 
- \beta_1(\lambda_{p+1})\|J_p\|_2^2 
+ \beta_2(\lambda_{p+1})\sigma^2_{p},
\end{align}
which completes the proof.
\end{proof}

\begin{lemma}[Lemma \ref{lem:ub-tgtloss}]
    Under Assumptions~\ref{asp:one-gd} and~\ref{asp:conv-smooth}, and the cross-client statistics defined in Definition ~\ref{def:cc-stat}, after $P$ rounds of federated pretraining one has
    \begin{align}
        \label{eq:Mrd-tdf-fed-ub-rstt}
        \widehat{L}^*_{tgt}\leq L_{src}\left(h_{0}\right)
        -\sum_{p=0}^{P-1}
        \beta_1({\lambda}_{p+1})
        \|J_p\|_2^2 + \sum_{p=0}^{P-1}
        \beta_2({\lambda}_{p+1})
        \sigma^2_{p} + d_{\gG, \gF}(\gD_S^{fed},\gD_T),
    \end{align}
    where $h_{0}$ is the initial global model and $\beta_2(\lambda) = \frac{\alpha{\lambda}^2}{2},\ \beta_1(\lambda) = {\lambda} - \beta_2(\lambda)$. 
\end{lemma}

\begin{proof}
From Lemma \ref{lem:round-ub}, we have
\begin{align*}
    L_{src}\left(h_{P}\right) &\leq L_{src}\left(h_{P-1}\right) 
- \beta_1(\lambda_{P})\|J_{p-1}\|_2^2 
+ \beta_2(\lambda_{P})\sigma^2_{P-1}\\
& \leq  L_{src}\left(h_{0}\right)
        -\sum_{p=0}^{P-1}
        \beta_1({\lambda}_{p+1})
        \|J_p\|_2^2 + \sum_{p=0}^{P-1}
        \beta_2({\lambda}_{p+1})
        \sigma^2_{p}.
\end{align*}

Following the same proof as the one outlined for Theorem~\ref{thrm:gen-ub}, we obtain
\begin{align*}
    L_{tgt}(h_{P}) \leq L_{src}(h_{P}) + d_{\gG, \gF}(\gD_S^{fed}, \gD_T).
\end{align*}
Therefore,
\begin{align*}
    \widehat{L}^*_{tgt} = L_{tgt}(h_{P}) & \leq L_{src}(h_{P}) + d_{\gG, \gF}(\gD_S^{fed}, \gD_T)\\
    & \leq L_{src}\left(h_{0}\right)
        -\sum_{p=0}^{P-1}
        \beta_1({\lambda}_{p+1})
        \|J_p\|_2^2 + \sum_{p=0}^{P-1}
        \beta_2({\lambda}_{p+1})
        \sigma^2_{p}+ d_{\gG, \gF}(\gD_S^{fed}, \gD_T).
\end{align*}
\end{proof}


\section{Additional Experimental  Results}
\label{apdx:addi-exp}
{\textbf{Computing resources}. We used an NVIDIA GeForce RTX 3090 Ti GPU with a memory of 24247MiB. One run of experimental evaluation takes approximately $6$ hours for CIFAR10 $\to$ SVHN for the small client setting.}




\textbf{Fraction of Clients Participating in FL Rounds}. Besides the results generated using $20\%$ of the participating clients reported in the main text, we also report the results pertaining to $10\%$ and $100\%$ participating clients in Table~\ref{tab:ptcp}. A larger fraction of participants helps with improving transferability, i.e., as expected,  $100\%$ participation outperforms the setting with $10\%$ participation. However, we note that for $20\%$ of participating clients we already reach a performance comparable to that involving $100\%$ of the participants.

\begin{table}[ht]
    \centering
\begin{tabular}{cccc}
\hline
\cline{1-4}
Method& Coefficient & { $10\%$ participants} & { $100\%$ participants} \\
\cline{1-4}
{FedAVG}  &  NA & 0.52 $\pm$ 0.01 & 0.56 $\pm$ 0.03 \\
  & 1e-4 & 0.52$\pm$ 0.02& { {0.58}}$\pm$ 0.01\\
FedGTST & 5e-4 & \textbf{{0.58} $\pm$ 0.04} & \textbf{ 0.63$\pm$ 0.05}  \\ 
 & 1e-3 & { 0.53}$\pm$ 0.01 & { 0.60}$\pm$ 0.06\\
 \hline
\cline{1-4}
\end{tabular}
\vspace{0.04in}
\caption{Transferability versus the fraction of participating clients in each round. We report the results for CIFAR10 $\to$ SVHN on the LeNet backbone, with $K=100$.}
\label{tab:ptcp}
\end{table}

\textbf{Convergence Results.} Convergence results are plotted in Figure\blue{~\ref{fig:cvg}}. We observe that a model pretrained via FedGTST \blue{not only offers} better transferability than baselines, but also tends to converge faster during the finetuning stage (i.e., the green lines in all plots always converge faster than the grey dashed lines).

\begin{figure*}[ht]
    \centering
    \includegraphics[height=1.9in]{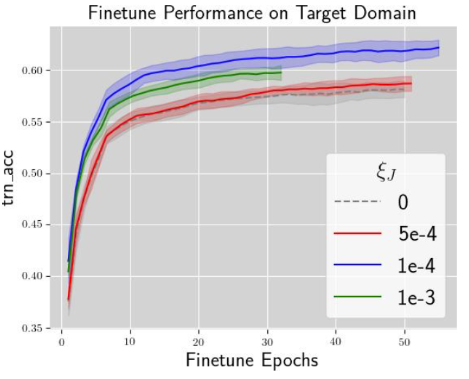}
    \includegraphics[height=1.9in]{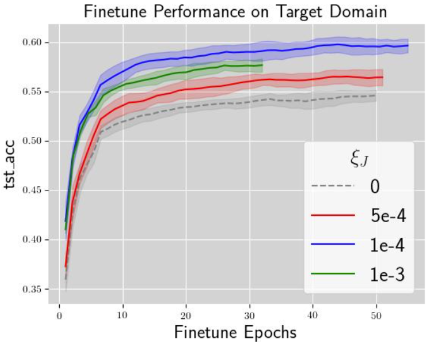}
     \includegraphics[height=1.9in]{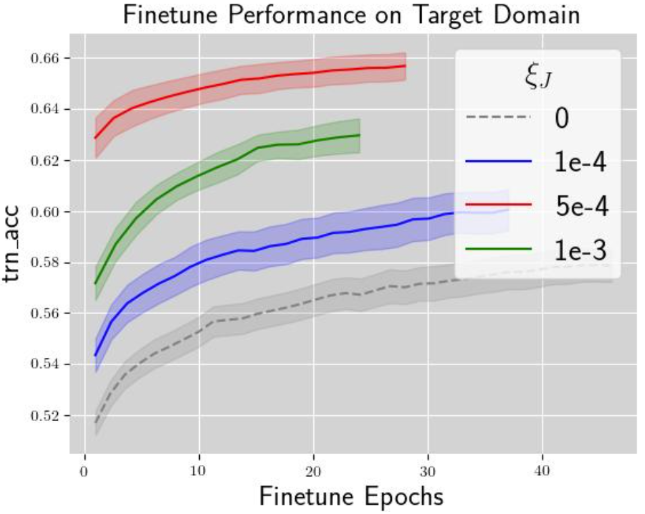}
    \includegraphics[height=1.9in]{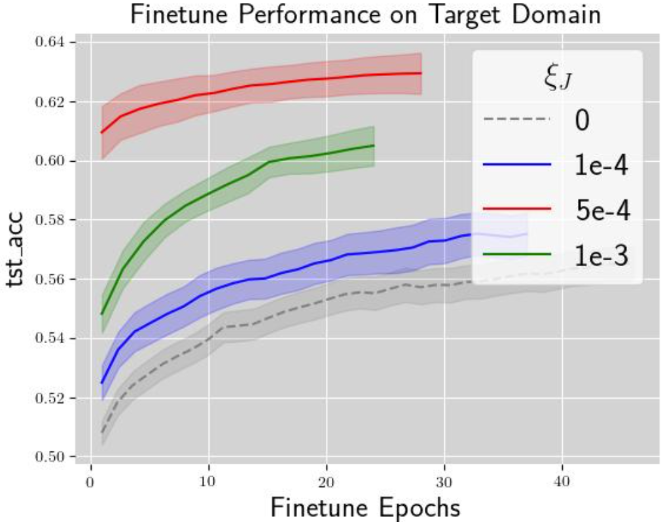}
    \caption{Visualization of Convergence Results. We use CIFAR10 $\to$ SVHN with $K=100$ as an example. The top two plots correspond to a fraction of $10\%$ of participating clients, while the bottom two plots correspond to $100\%$ participation. We report the training and test accuracy along with finetuned epochs for both settings. The grey dashed lines represent FedAVG, where the coefficient for the regularizer term is set to $0$. Other lines represent FedGTST with tuned coefficients.}
    \label{fig:cvg}
\end{figure*}

\textbf{Local Epochs}. Besides the results for single local client epochs presented in main text, here we also report the results for multiple local epochs in Table ~\ref{tab:le}. A smaller number of local epochs tends to improve transferability more than a larger number of epochs, which is consistent with the intuition that FL pretraining can avoids local overfitting when using fewer epochs.

\begin{table}[ht]
    \centering
\begin{tabular}{cccc}
\hline
\cline{1-4}
  &Coeffient & 1 local epoch & 10 local epochs\\
\cline{1-4}
FedAVG&NA&0.54$\pm$0.05 & 0.52$\pm$0.07\\
 & 1e-4  &\textbf{0.60$\pm$0.01} & 0.52$\pm$0.04\\
FedGTST & 5e-4  & 0.56$\pm$0.06& {\textbf{0.58$\pm$0.02}} \\
 & 1e-3  & 0.58$\pm$0.05& { 0.53$\pm$0.02}\\
\hline
\cline{1-4}
\end{tabular}
\vspace{0.04in}
\caption{Transferability v.s. local number of epochs. We report results for CIFAR10$\to$SVHN with the LeNet backbone, for $10\%$ of participating clients and $K=100$.}
    \label{tab:le}
\end{table}

\textbf{Cross-Client Statistics}. We plot the cross-client averaged Jacobian norm $\|J_p\|_2$ and the cross-client Jacobian variance $\sigma_p^2$ in Figure~\ref{fig:ccs}. The observation is that FedGTST leads to a significantly larger $\|J_p\|_2$ and a significantly smaller $\sigma_p^2$ compared to FedAVG. A more detailed explanation of the findings (i.e., the reason for truncating the x-axis in the left plot, the FedGTST coefficient selection approach and faster convergence results in the right plot) is available in the caption of Figure~\ref{fig:ccs}.
\begin{figure*}[ht]
    \centering
    \includegraphics[height=1.9in]{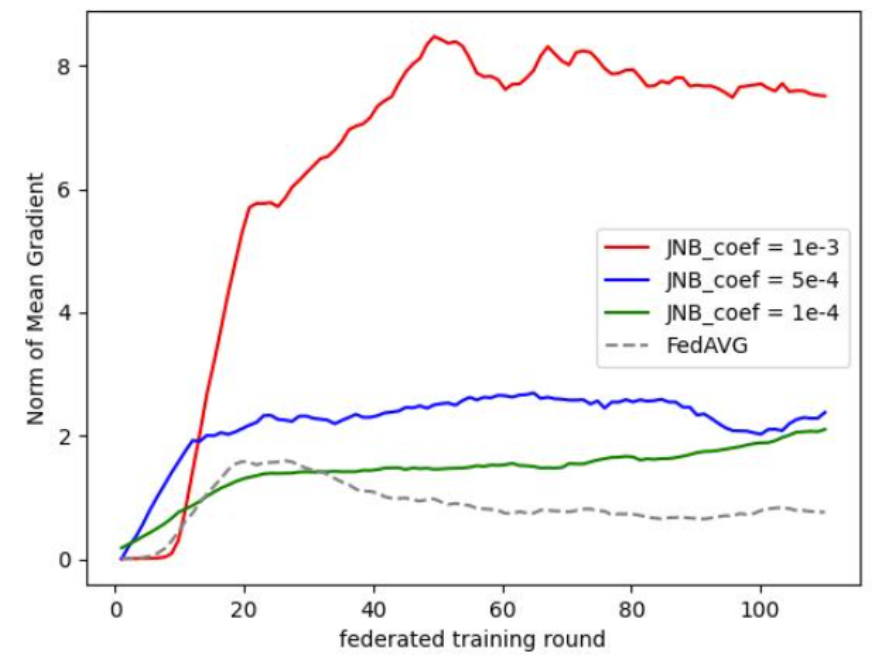}
    \includegraphics[height=1.96in]{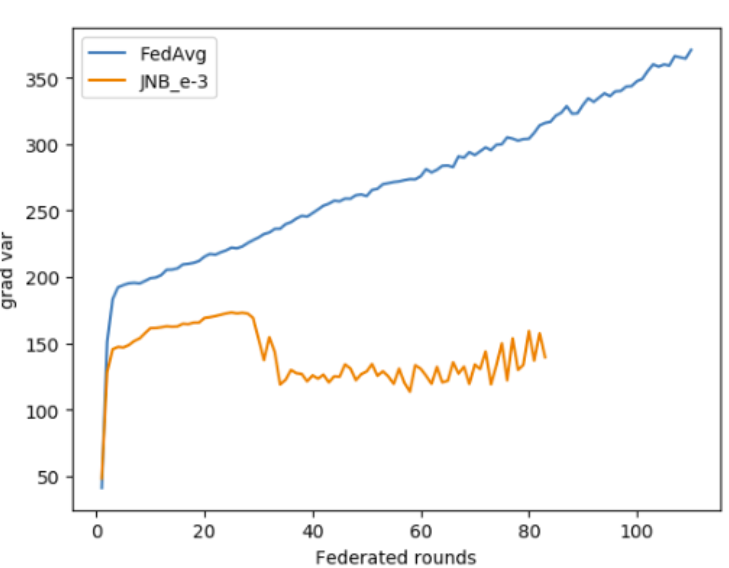}
    \caption{Cross-client statistics tuning via FedGTST. We use CIFAR10$\to$SVHN with $K=100$ as an example. The left plot reports the global Jacobian (gradient) norm versus the index of the federated round. The grey dashed line represents FedAVG, while other lines correspond to FedGTST with different coefficients.
    We truncate the plot to only capture the results of the first $100$ rounds, since at the end of training the gradient norm should drop to a value close to $0$ due to convergence, and we are only interested in observing the behaviour of Jacobian norms during relative early pretraining stages. 
    We select the best-performing setup from the left plot (the red line with coefficient $1e-3$), and then in the right plot, compare its variance during a federated round with that of FedAVG. The blue line represents FedAVG and the yellow line corresponds to FedGTST. The yellow line terminated earlier since all experiments are averaged over $3$ runs and aligned with the run that converges the fastest.}
    \label{fig:ccs}
\end{figure*}

\blue{\textbf{Additional Dataset: DomainNet}. Following FedSR, we apply a leave-one-out strategy, where one domain is treated as the target domain and the other five are source domains, allocated to five individual clients.
Results are reported in Table \ref{tab:DomainNet}, where the target domains listed are: C (Clipart), I (Infograph), P (Painting), Q (Quickdraw), R (Real), and S (Sketch).}

\begin{table}[ht]
\centering
\begin{tabular}{ccccccc}
\hline
\cline{1-7}

& \multicolumn{6}{c}{{Leave-One-Out Domain}} \\ 
\cline{2-6}
\multirow{-2}{*}{Model} & \multicolumn{1}{c}{{ C}} & \multicolumn{1}{c}{{ I}} & \multicolumn{1}{c}{{ P}} & \multicolumn{1}{c}{{ Q}} & \multicolumn{1}{c}{{ R}} & \multirow{-2}{*}{Average} \\ 
\hline
{ FedAVG}   & { 59.3$\pm$0.7}  & { 16.5$\pm$0.9} & { 44.2$\pm$0.7} & { 10.8$\pm$1.8} & { 57.2$\pm$0.8} & 39.6 \\
{ FedSR}    & { 61.0$\pm$0.6}  & { 18.6$\pm$0.4} & { 45.2$\pm$0.5} & { 13.4$\pm$0.6} & { 57.6$\pm$0.2} & 41.3 \\
{ FedGTST}  & { \textbf{63.9$\pm$0.5}} & { \textbf{20.7$\pm$0.3}} & { \textbf{47.8$\pm$0.4}} & { \textbf{15.2$\pm$0.5}} & { \textbf{59.5$\pm$0.6}} & \textbf{43.6} \\ 
\hline
\cline{1-7}
\end{tabular}
\vspace{0.04in}
\caption{\blue{Comparison of different federated models on intra-domain transfer tasks of DomainNet. FedGTST consistently outperforms the other methods.}}
\label{tab:DomainNet}
\end{table}

\blue{\textbf{Additional Baseline: Scaffold}. FedGTST consistently outperforms Scaffold, as presented in Table \ref{tab:scaffold-10} and \ref{tab:scaffold-100}.}

\begin{table}[ht]
\centering
\begin{tabular}{cccccc}
\hline
\cline{1-6} 
& \multicolumn{2}{c}{{ MNIST→ MNIST-M}} & \multicolumn{2}{c}{{ CIFAR10 → SVHN}}&\\ 
\cline{2-5}
\multirow{-2}{*}{Method}& { LeNet}& { ResNet}& { LeNet}& { ResNet}& \multirow{-2}{*}{Average} \\ 
\hline
{ Scaffold} & { 75.6} $\pm$ 0.8 & { 80.8} $\pm$ 0.3 & { 66.0} $\pm$ 0.5 & { 71.1} $\pm$ 0.4 & 73.3 \\
{ FedGTST} & { \textbf{76.2} $\pm$ 0.9} & { \textbf{82.3} $\pm$ 0.5} & { \textbf{70.1} $\pm$ 0.8} & { \textbf{74.5} $\pm$ 0.3} & \textbf{75.8} \\ 
\hline
\cline{1-6}
\end{tabular}
\vspace{0.04in}
\caption{\blue{Comparison between Scaffold and FedGTST. Number of clients is 10.}}
\label{tab:scaffold-10}
\end{table}

\begin{table}[ht]
\centering
\begin{tabular}{cccccc}
\hline
\cline{1-6} 
& \multicolumn{2}{c}{{ MNIST→ MNIST-M}} & \multicolumn{2}{c}{{ CIFAR10 → SVHN}}&\\ 
\cline{2-5}
\multirow{-2}{*}{Method}& { LeNet}& { ResNet}& { LeNet}& { ResNet}& \multirow{-2}{*}{Average} \\ 
\hline
{ Scaffold} & { 52.3} $\pm$ 0.5 & { 63.1} $\pm$ 0.3 & { 45.5} $\pm$ 0.1 & { 55.5} $\pm$ 0.3 & 54.1 \\
{ FedGTST} & { \textbf{57.5} $\pm$ 0.3} & { \textbf{67.6} $\pm$ 0.2} & { \textbf{52.4} $\pm$ 0.1} & { \textbf{63.1} $\pm$ 0.2} & \textbf{60.2} \\ 
\hline
\cline{1-6}
\end{tabular}
\vspace{0.04in}
\caption{\blue{Comparison between Scaffold and FedGTST. Number of clients is 100.}}
\label{tab:scaffold-100}
\end{table}

\blue{\textbf{Dirichlet Sampling}. We set the concentration parameter to 0.5 and the number of parties to 10 by default. The results in Table \ref{tab:noniid-10} and \ref{tab:noniid-100} indicate that FedGTST still outperforms others when individual domains are constructed via Dirichlet sampling.}

\begin{table}[ht]
\centering
\begin{tabular}{cccccc}
\hline
\cline{1-6} 
& \multicolumn{2}{c}{{ MNIST→ MNIST-M}} & \multicolumn{2}{c}{{ CIFAR10 → SVHN}}&\\ 
\cline{2-5}
\multirow{-2}{*}{Method}& { LeNet}& { ResNet}& { LeNet}& { ResNet}& \multirow{-2}{*}{Average} \\ 
\hline
{ FedAVG}& { 74.1}$\pm$0.6 & { 82.2}$\pm$0.3& { 65.2}$\pm$0.4& { 72.8}$\pm$0.9 & 73.5 \\
{ FedSR}& { 75.5}$\pm$0.8 & { 82.0}$\pm$0.2 & { 66.1}$\pm$0.5 & { 72.1}$\pm$0.3 & 73.9 \\
{ FedIIR} & { 75.8}$\pm$0.2& { \textbf{82.8}$\pm$0.6}& { 66.5}$\pm$0.9& { 74.6}$\pm$0.1& 74.9\\
{ Fed-GTST} & { \textbf{77.3}$\pm$0.8} & { 82.7}$\pm$0.4 & { \textbf{70.8}$\pm$0.7} & { \textbf{75.2}$\pm$0.2} & \textbf{76.5} \\ 
\hline
\cline{1-6}
\end{tabular}
\vspace{0.04in}
\caption{\blue{Comparison of target accuracy (\%) across different methods on MNIST→MNIST-M and CIFAR10→SVHN tasks. 10 individual domains are constructed by Dirichlet sampling.}}
\label{tab:noniid-10}
\end{table}

\begin{table}[ht]
\centering
\begin{tabular}{cccccc}
\hline
\cline{1-6} 
& \multicolumn{2}{c}{{ MNIST→ MNIST-M}} & \multicolumn{2}{c}{{ CIFAR10 → SVHN}}&\\ 
\cline{2-5}
\multirow{-2}{*}{Method}& { LeNet}& { ResNet}& { LeNet}& { ResNet}& \multirow{-2}{*}{Average} \\ 
\hline
{ FedAVG} & { 50.4}$\pm$0.1 & { 63.0}$\pm$0.3 & { 43.3}$\pm$0.2 & { 54.6}$\pm$0.5 & 52.8 \\
{ FedSR} & { 52.7}$\pm$0.2 & { 62.9}$\pm$0.3 & { 44.7}$\pm$0.1 & { 56.5}$\pm$0.3 & 54.2 \\
{ FedIIR} & { 54.2}$\pm$0.4 & { 64.1}$\pm$0.1 & { 47.4}$\pm$0.4 & { 58.7}$\pm$0.2 & 56.1 \\
{ Fed-GTST} & { \textbf{59.5}$\pm$0.3} & { \textbf{69.2}$\pm$0.2} & { \textbf{55.1}$\pm$0.1} & { \textbf{65.6}$\pm$0.2} & \textbf{62.4} \\ 
\hline
\cline{1-6}
\end{tabular}
\vspace{0.04in}
\caption{\blue{Comparison of target accuracy (\%) across different methods on MNIST→MNIST-M and CIFAR10→SVHN tasks. 100 individual domains are constructed by Dirichlet sampling.}}
\label{tab:noniid-100}
\end{table}


\section{Discussion of the Positivity of Coefficients Requirement for Cross-Client Statistics}
\label{apdx:pstv-coef}
\begin{lemma}[Bound on round-wise source loss]
\label{lem:round-ub-rstt}
Suppose the learning rates of all clients at round $p$ are the same: $\lambda^{(k)}_p = {\lambda}_{p},\forall k\in[K], p\in[P]$. When Assumption~\ref{asp:one-gd} and~\ref{asp:conv-smooth} hold, we have that
\begin{align}
\label{eq:round-ub-rstt}
L_{src}\left(h^{}_{p+1}\right)\leq L_{src}\left(h^{}_{p}\right)
    - \beta_1({\lambda}_{p+1}) \|J_p\|_2^2 + 
     \beta_2({\lambda}_{p+1}) \sigma^2_{p}
\end{align}
where 
$
    J_p = \frac{1}{K}\sum_{k} J^{(k)}_p,\ 
    \sigma^2_{p} = \frac{1}{K}\sum_{k} 
    \left\|J^{(k)}_p\right\|_2^2 - \left\|\frac{1}{K}\sum_{k} J^{(k)}_p\right\|_2^2,
$
and $\beta_1(\lambda) = {\lambda} - \beta_2(\lambda),\ \beta_2(\lambda) = \frac{\alpha{\lambda}^2}{2}$.

\end{lemma}

{\remark Lemma \ref{lem:round-ub-rstt} provides an upper bound of  \textit{current-round cross-client loss} (the proof is given in Appendix~\ref{app:lem:round-ub}), indicating that a smaller \textit{current-round cross-client loss} (LHS) may be influenced by factors such as a small \textit{last-round cross-client loss}, a large \textit{cross-client average Jacobian norm}, and a small \textit{cross-client Jacobian variance} (RHS).
More precisely, the LHS is the \textit{cross-client current-round loss}, which is upper bounded by three terms on the RHS, 

1) \textit{The Loss Term}: $L_{src}\left(h_{p}^{(0)}\right)$ is the \textit{last-round cross-client loss}.

2) \textit{The Variance Term}: $\sigma^2_{p}$ measuring the variance of local gradients across all nodes. 

3) \textit{The Norm Term}: $\|J_p\|^2_2$ measures the squared \textit{cross-client average Jacobian norm}. 

\textbf{Positivity of the Coefficients}.
It is straightforward to see that $\beta_2({\lambda}_{p+1})=\frac{\alpha(\lambda_{p+1})^2}{2}$, the coefficient in front of $\sigma_p^2$, is always positive. Therefore, we focus on  $\beta_2({\lambda}_{p+1})$, the coefficient in front of $\|J_p\|_2$. Denote the RHS of Equation~\ref{eq:round-ub-rstt} by $\gB$, so that $L_{src}(h_{p+1}) \leq \gB$. To allow such a bound to be used as an indicator that the source loss is decreasing with the number of training rounds, we require $\gB \leq L_{src}(h_p)$, since only in this way can the bound give $L_{src}(h_{p+1}) \leq \gB \leq L_{src}(h_p)$. Thus, we require $\beta_1(\lambda_{p+1})$ to be positive, since otherwise $\gB \leq L_{src}(h_p)$ cannot be meet. We describe in what follows that a positive $\beta_1(\lambda_{p+1})$ is indeed possible in practice.

\textbf{Realistic scenarios in which $\beta_1(\lambda_{p+1})>0$}.} To require $\beta_1(\lambda_{p+1})>0$ is equivalent to require $\lambda_{p+1} < \frac{2}{\alpha}$. This requirement can be easily met since for any model of a known mathematical form based on a second-order differentiable loss, we can easily get the lower bound for $\alpha$ once we observed all training data. Using linear regression as an example, where $l(x,y;w) = (wx-y)^2$, we have $\frac{d^2l}{dw^2} = 2x$, therefore, in this case we can simply control $\lambda_{p+1}\leq \frac{1}{\max_{x\in \gX}\|x\|}$ to approximately guarantee $\lambda_{p+1} < \frac{2}{\alpha}$. This then ensures $\beta_1(\lambda_{p+1})>0$.

\section{Optimal Learning Rate}
\textit{Choosing the Optimal Learning Rate}. Lemma~\ref{lem:round-ub} shows that the upper bound of federated loss at round $p$ is a quadratic function w.r.t the learning rate $\lambda_{p}$. Therefore, a good learning rate at each round needs to be chosen to minimize the upper bound. For simplicity of notation, we use $B_{p+1}(\lambda_{p+1})$ as a shorthand for the upper bound shown in~(\ref{eq:round-ub}).

The following two observations are in place for $B_{p+1}(\lambda_{p+1})$:
\begin{itemize}
    \item When $0 < \lambda_{p+1} < \frac{2K\cdot\|J_p\|_2^2}{\alpha \cdot \sum_k\|J_p^{(k)}\|_2^2} \leq \frac{2}{\alpha}$, it holds that $L_{src}\left(h_{p+1}^{(0)}\right) \leq B_{p+1}(\lambda_{p+1}) < L_{src}\left(h_{p}^{(0)}\right)$, indicating that the federated loss is decreasing with the number of rounds.

    \item By minimizing $B_{p+1}(\lambda_{p+1})$ w.r.t. $\lambda_{p+1}$, we have 
    \begin{align*}
        \lambda_{p+1}^* = \frac{K\cdot\|J_p\|_2^2}{\alpha \cdot \sum_k\|J_p^{(k)}\|_2^2},\quad B_{p+1}(\lambda_{p+1}^*) = L_{src}\left(h_p^{(0)}\right)-\frac{K \cdot \|J_p\|_2^4}{2\alpha\cdot \sum_k\|J_p^{(k)}\|_2^2}.
    \end{align*}
\end{itemize}

\section{\blue{Stochastic Gradient Descend}}
\label{sec:apdx-sgd}
\blue{Assumption \ref{asp:one-gd} on one step of gradient descent can be extended to stochastic learning with batch sampling. The additional randomness in sampling would require incorporating the variance of batch sampling into the generalization bound. This variance term, being independent of the algorithm design, was omitted in our theoretical analysis.}

\section{Dataset Description}
\label{apdx:ds-intro}
MNIST comprises $60,000$ $28\times28$ grayscale images of handwritten digits (0 through 9); MNIST-M is created by combining MNIST digits with the patches randomly extracted from color photos of BSDS500 as their background, containing $59,001$ training images.
The CIFAR-10 dataset consists of $60,000$ $32\times32$ colour images from $10$ classes. SVHN contains $73,257$ $32\times32$ colored digits obtained from house numbers in Google Street View images.

\section{Non-iid FL Models}
\label{apdx:FL-noniid}
For the source domains MNIST and CIFAR-10, we only allow each local client to have access to two out of ten classes (e.g., for the digit dataset (0-9), one client may only have access to say digits 3 and 6). We let each client randomly chooses their labels and samples following a uniform distribution. See the example in Fig.~\ref{fig:label-dist} 
\begin{figure}[ht]
    \centering
    \includegraphics[width=0.3\linewidth]{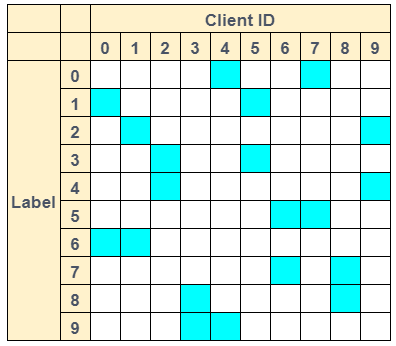}
    \caption{An example for constructing a non-iid marginal distribution for the  \textit{Cifar10} dataset allocated to $10$ clients. Each client has access to only two labels. We also make sure that no samples are used by more than one client.}
    \label{fig:label-dist}
\end{figure}


\section{\blue{Additional Related Work}}
\subsection{\blue{Gradient Matching in Transfer Learning}}
\label{subsec:apdx-lit-grad-match}
   \blue{\textbf{Discussion.} Gradient matching in transfer learning focuses on aligning gradients between source and target domains to facilitate better domain adaptation. \citet{ShiSTNHUS22} demonstrate that matching gradients across domains enhances domain generalization by making the learned representations more resilient to domain shifts. Extending this idea,~\citet{pmlr-v162-rame22a} introduce the concept of invariant gradient variances, which helps maintain generalization performance even for out-of-distribution settings. The work in~\citet{pezeshki2022understanding} further highlights the impact of Hessian alignment, showing that aligning the Hessians of the source and target domains can significantly boost generalization in gradient-based methods.}

\blue{\textbf{Challenges.} Applying gradient alignment techniques directly to FL presents several challenges: (1) \textit{Privacy leakage}—these methods can potentially compromise data privacy by necessitating access to the target domain from source domains; (2) \textit{Local overfitting}—clients train their models on local data, which can lead to overfitting within their specific domains, reducing the global model's generalization capabilities.}

\blue{\textbf{Our contribution.} To address the issues described above, we propose communication schemes that ensure data privacy by preventing unrestricted access to client data. Furthermore, our approach promotes global transferability by focusing on improving generalization across all clients rather than solely enhancing local domain performance.}


\subsection{\blue{Distinctions and Connections between Generalization and Transferability in FL}}
\label{subsec:apdx-lit-GFL}

\blue{Both approaches address the challenge of non-iid data,
and improving transferability may potentially enhance generalization across diverse local domains. However, they employ distinct models and evaluation datasets: (a) Generalization of FL targets performance on heterogeneous \textit{source testing datasets}, while transferable FL aims for strong performance on a \textit{target dataset} that may significantly differ from \textit{source training datasets}. (b) Heterogeneous FL utilizes the pretrained model for evaluation, whereas Transferable FL assesses the finetuned model.}

\blue{The methodologies also diverge. Although reducing cross-client variance is linked to better generalization, our method distinguishes itself from traditional heterogeneous FL by enforcing a large average Jacobian norm $|J_p|$ in the early stages. While a larger $|J_p|$ may hinder generalization due to increased local updates and model variance, it enhances transferability by preventing premature local convergence.}

\blue{Additionally, adversarially robust models offer an example where improving transferability may come at the expense of generalization. As shown in \citep{osti_10283431}, adversarially robust models tend to have better transferability. However, since these models are designed to perform well against adversarial examples or perturbations, they do not necessarily exhibit lower generalization error on clean test data. In this scenario, transferability can be unrelated to or even conflict with generalization.}

\newpage
\end{document}